\documentclass{article}
\usepackage[preprint]{neurips_2019}

\usepackage[utf8]{inputenc} % allow utf-8 input
\usepackage[T1]{fontenc}    % use 8-bit T1 fonts
\usepackage{hyperref}       % hyperlinks
\usepackage{url}            % simple URL typesetting
\usepackage{booktabs}       % professional-quality tables
\usepackage{amsfonts}       % blackboard math symbols
\usepackage{nicefrac}       % compact symbols for 1/2, etc.
\usepackage{microtype}      % microtypography
\usepackage{natbib}
\usepackage{color}
\usepackage{amsmath}
\usepackage{amssymb}
\usepackage{amsthm}
\usepackage{bbm}
\usepackage{mathtools}
\usepackage{ifthen}
\usepackage{subcaption}
\DeclarePairedDelimiter{\ceil}{\lceil}{\rceil}
\DeclarePairedDelimiter{\floor}{\lfloor}{\rfloor}

\usepackage{algorithm}
\usepackage{algpseudocode}

\newtheorem{theorem}{Theorem}

\newtheorem{lemma}{Lemma}

\newcommand{\cvargen}{\hat{c}_{n,\alpha}^{\dagger}}
\newcommand{\cvarte}{\hat{c}_{n,\alpha}^{(b)}}
\newcommand{\valp}{v_{\alpha}}
\newcommand{\calp}{c_{\alpha}}
\newcommand{\valpn}{\hat{v}_{n,\alpha}}
\newcommand{\calpn}{\hat{c}_{n,\alpha}}
\newcommand{\calpt}{\hat{c}_{n,\alpha}^{\dagger}}
\newcommand{\calptk}{\hat{c}_{n_{k},\alpha}^{\dagger}}
\newcommand{\knb}{K_{n, \beta}}
\newcommand{\pbb}{\mathbb{P}}
\newcommand{\cnb}{\ceil{n\beta}}
\newcommand{\fnb}{\floor{n\beta}}
\newcommand{\ksg}{k^{*}_{\gamma}}
\newcommand{\tea}{\hat{\mu}^{\dagger}}

\newcommand{\meantea}{\mu^{\dagger}}

\newcommand{\prob}[1]{\mathsf{Pr}\left( #1 \right)}

\newcommand{\remove}[1]{}

\newcommand{\ra}{\rightarrow}

\newcommand{\Exp}[1]{\mathbb{E}\left[#1\right]}
\newcommand{\Ind}[1]{\mathbbm{1}\left\{#1\right\}}

\newcommand{\ignore}[1]{}
\newcommand{\nn}{\nonumber}

\newboolean{showcomments}
\setboolean{showcomments}{true}
\newcommand{\ak}[1]{  \ifthenelse{\boolean{showcomments}}
{ \textcolor{red}{(AK says:  #1)}} {}  }
\newcommand{\jk}[1]{  \ifthenelse{\boolean{showcomments}}
{ \textcolor{red}{(JK says:  #1)}} {}  }
\newcommand{\kj}[1]{  \ifthenelse{\boolean{showcomments}}
{ \textcolor{red}{(KJ says:  #1)}} {}  }
\newcommand{\addcites}[0]{\ifthenelse{\boolean{showcomments}}
{ \textcolor{green}{(add citation(s))}}{}}
\newcommand{\addref}[0]{\ifthenelse{\boolean{showcomments}}
{ \textcolor{green}{(add ref)}}{}}

\title{Distribution oblivious, risk-aware algorithms for multi-armed
  bandits with unbounded rewards}

\author{%
 Anmol Kagrecha \\ EE, IIT Bombay \\ 
 \And Jayakrishnan Nair \\ EE, IIT Bombay \\
 \And Krishna Jagannathan \\ EE, IIT Madras \\
}

\begin{document}

\maketitle

\begin{abstract}
Classical multi-armed bandit problems use the expected value of an arm
as a metric to evaluate its goodness. However, the expected value is a
risk-neutral metric. In many applications like finance, one is
interested in balancing the expected return of an arm (or portfolio)
with the risk associated with that return. In this paper, we consider
the problem of selecting the arm that optimizes a linear combination
of the expected reward and the associated Conditional Value at Risk
(CVaR) in a fixed budget best-arm identification framework. We allow
the reward distributions to be unbounded or even heavy-tailed. For
this problem, our goal is to devise algorithms that are entirely
\emph{distribution oblivious,} i.e., the algorithm is not aware of any
information on the reward distributions, including bounds on the
moments/tails, or the suboptimality gaps across arms.

In this paper, we provide a class of such algorithms with provable
upper bounds on the probability of incorrect identification. In the
process, we develop a novel estimator for the CVaR of unbounded
(including heavy-tailed) random variables and prove a concentration
inequality for the same, which could be of independent interest. We
also compare the error bounds for our distribution oblivious
algorithms with those corresponding to standard non-oblivious
algorithms.
%highlighting the cost of distribution obliviousness.
Finally, numerical experiments reveal that our algorithms perform
competitively when compared with non-oblivious algorithms, suggesting
that distribution obliviousness can be realised in practice without
incurring a significant loss of performance.
\end{abstract}

\section{Introduction}

The multi-armed bandit (MAB) problem is fundamental in online
learning, where an optimal option needs to be identified among a pool
of available options. Each option (or arm) generates a random
reward/cost when chosen (or pulled) from an underlying unknown
distribution, and the goal is to quickly identify the optimal arm by
exploring all possibilities.

Classically, MAB formulations consider reward distributions with
bounded support, typically $[0,1].$ Moreover, the support is assumed
to be known beforehand, and this knowledge is baked into the
algorithm. However, in many applications, it is more natural to not
assume bounded support for the reward distributions, either because
the distributions are themselves unbounded, or because a bound on the
support is not known {\it a priori}. There is some literature on MAB
formulations with (potentially) unbounded rewards; see, for example,
\cite{bubeck2013,vakili2013}. Typically, in these
papers, the assumption of a known bound on the support of the reward
distributions is replaced with the assumption that certain bounds on
the moments/tails of the reward distributions are
known.\footnote{Additionally, many algorithms require knowledge of a
  lower bound on the sub-optimality gap between arms.}  However, such
access to prior information is not always practical, and goes against
the spirit of {\it online} learning. This motivates the design and
analysis of algorithms for the MAB problem that are \emph{distribution
  oblivious}, i.e., algorithms that have zero prior knowledge about
the reward distributions.

Furthermore, the typical metric used to quantify the goodness of an
arm in the MAB framework is its expected return, which is a
risk-neutral metric. In some applications, particularly in finance,
one is interested in balancing the expected return of an arm with the
risk associated with that arm. This is particularly relevent when the
underlying reward distributions are unbounded, even heavy-tailed, as
is found to be the case with portfolio returns in finance
\citep{bradley2003}. In these settings, there is a non-trivial
probability of a `catastrophic' outcome, which motivates a risk-aware
approach to optimal arm selection.

In this paper, we seek to address the two issues described
above. Specifically, we consider the problem of identifying the arm
that optimizes a linear combination of the reward and the Conditional
Value at Risk (CVaR) in a fixed budget (pure exploration) MAB
framework. The CVaR is a classical metric used to capture the risk
associated with an option/portfolio \citep{artzner1999}. We make very
mild assumptions on the reward distributions (the existence of a
$(1+\epsilon)$th moment for some~$\epsilon > 0$), allowing for
unbounded support and even heavy tails. In this setting, our goal is
to design algorithms that are entirely distribution oblivious.

The main contribution of this paper is the design and analysis of
distribution oblivious algorithms for the risk-aware best arm
identification problem described above. These algorithms are based on
truncation-based estimators for the mean and CVaR, where the
truncation parameters are scaled suitably as the algorithm runs. We
prove upper bounds on the probability of incorrect arm identificiation
for these algorithms that have the form
$O(\mathrm{exp}(-\gamma T^{1-q})),$ where $T$ is the budget of arm
pulls, $\gamma > 0$ is a constant that depends on the arm
distributions, and $q \in (0,1)$ is an algorithm parameter. Note the
slower-than-exponential decay in the probability of erronious arm
identification with respect to~$T.$ This is a consequence of the
distribution obliviousness of the proposed algorithms. Indeed, in the
non-oblivious setting, it is easy to develop algorithms with an
$O(\mathrm{exp}(-\gamma' T))$ probability of error. Moreover,
numerical experiments show that the proposed distribution oblivious
algorithms perform competitively when compared with standard
non-oblivious algorithms. This suggests that distribution
obliviousness can be realised in practice without incurring a
significant performance hit.% (while admittedly also suggesting that
                            % our performance bounds are somewhat
                            % loose).

Finally, we note that the truncation-based CVaR estimator used in our
algorithms is novel, and the concentration inequality we prove for
this estimator may be of independent interest.

The remainder of this paper is organized as follows.  A brief survey
of the related literature is provided below, followed by some
preliminaries. Our CVaR concentration results are presented in
Section~\ref{sec:cvar_concentration}, and our distribution oblivious
algorithms for risk-aware best arm identification are proposed and
analysed in Section~\ref{sec:algorithms}. Numerical experiments are
presented in Section~\ref{sec:experiments}, and we conclude in
Section~\ref{sec:discuss}. Throughout the paper, references to the
appendix (primarily for proofs) point to the `additional material'
document uploaded separately.

%\subsubsection*{Related Literature}
\textbf{Related Literature}

There is a considerable body of literature on the multi-armed bandit
problem. We refer the reader to \cite{bubeck2012} and
\cite{lattimore2018} for a comprehensive review. Here, we restrict
ourselves to papers that consider (i) unbounded reward distributions,
and (ii) risk-aware arm selection.

The papers that consider MAB problems with (potentially) heavy-tailed
reward distributions include: \cite{bubeck2013,vakili2013,
  boucheron2013}, which consider the regret minimization framework,
and \cite{yu2018}, which considers the pure exploration framework. All
the above papers take the expected return of an arm to be its goodness
metric. \cite{bubeck2013,vakili2013} assume prior knowledge of moment
bounds and/or the suboptimality gaps. \cite{boucheron2013} assumes
that the arms belong to parametrized family of distributions
satisfying a second order Pareto condition.
% , see \cite{de2007}.
\cite{yu2018} does analyse one distribution oblivious algorithm (see
Theorem 2 in the paper), though the performance guarantee derived
there is much weaker than the ones proved here; we elaborate on this
in Section~\ref{sec:algorithms}.

There has been some recent interest in risk-aware multi-armed bandit
problems. \cite{sani2012} considers the setting of optimizing a linear
combination of mean and variance in the regret minimization framework.
In the pure exploration setting, VaR-optimization has been considered
in \cite{david2016, david2018}. However, the CVaR is a more preferable
metric because it is a coherent risk measure (unlike the VaR); see
\cite{artzner1999}. Strong concentration results for VaR are available
without any assumptions on the tail of the distribution
\citep{kolla2019a}, whereas concentration results for CVaR are more
difficult to obtain. CVaR-optimization has only been considered before
by making restrictive assumptions on the reward distribution:
\cite{galichet2013} assumes bounded rewards, and \cite{kolla2019b}
assumes that the reward distributions are sub-exponential. None of the
above papers considers the problem of risk-aware arm selection
allowing for heavy-tailed reward distributions (much less in a
distribution oblivious fashion), as is done here.

\textbf{Preliminaries}

Here, we define the Value at Risk (VaR) and the Conditional Value at
Risk (CVaR), and state the assumptions we make in this paper on the
arm distributions.

For a random variable $X,$ given confidence level $\alpha \in (0,1)$,
the Value at Risk (VaR) is defined as $\valp(X) = \text{inf}(\xi: \pbb(X \leq \xi) \leq \alpha).$ If $X$ denotes the loss associated with a portfolio, $\valp(X)$ can be
interpreted as the worst case loss corresponding to the confidence
level $\alpha.$ The Conditional Value at Risk (CVaR) of $X$ at
confidence level $\alpha \in (0,1)$ is defined as 
%\begin{equation*}
$\calp(X) = \valp(X) + \frac{1}{1-\alpha}\mathbb{E}[X-\valp(X)]^{+},$
%\end{equation*}
where $[z]^{+} = $ max$(0,z)$. Going back to our portfolio loss
analogy, $\calp(X)$ can be interpreted as the expected loss
conditioned on the `bad event' that the loss exceeds the VaR. Both VaR
and CVaR are used extensively in the finance community as measures of
risk, through the CVaR is often preferred as mentioned
above. Typically, the confidence level $\alpha$ is chosen between
$0.95$ and $0.99$. Throughout this paper, we use the CVaR as a measure
of the risk associated with an arm. We define $\beta := 1-\alpha.$

For simplicity, we often assume the following condition: We say a
random variable $X$ satisfies condition \textbf{C1} if $X$ is
continuous with a strictly increasing cumulative distribution function
(CDF) over its support. If $X$ satisfies \textbf{C1}, then
$\valp(X) = F_X^{-1}(\alpha),$ where $F_X$ denotes the CDF of~$X.$

Finally, we require the following moment condition: A random variable
$X$ satisfies condition \textbf{C2} if there exists $p > 1$ and
$B < \infty$ such that $\Exp{|X|^p} < B.$ Note that \textbf{C2} is
only mildly more restrictive than assuming that the expectation of
$|X|$ is bounded. In particular, all light-tailed distributions and
most heavy-tailed distributions used and observed in practice satisfy
\textbf{C2}.

%%%%%%%%%%%%%%%%%%%%%%%%%%%%%%%%%%%%%%%%%%%%%%%%%%

\section{CVaR Concentration}
\label{sec:cvar_concentration}

In this section, we derive a concentration inequality for an estimator
of the CVaR corresponding to a distribution with unbounded
support. The key feature of this concentration inequality is that it
makes very mild assumptions on the tail of the distribution;
specifically, our concentration result applies even to heavy-tailed
distributions (unlike prior results in the literature, that assume a
bounded distribution \cite{wang2010}, or a subgaussian/subexponential
tail \cite{kolla2019a}). This CVaR concentration result
(Theorem~\ref{cvar-ht-theorem} below), while of independent interest,
will be invoked it in Section~\ref{sec:algorithms} to prove guarantees
on our algorithms for the risk-aware multi-armed bandit problem.

Assume that $\{X_{i}\}_{i=1}^{n}$ be $n$ i.i.d. samples distributed as
the random variable $X.$ Let $\{X_{[i]}\}_{i=1}^{n}$ denote the order
statistics of $\{X_{i}\}_{i=1}^{n}$ i.e., $X_{[1]} \geq
X_{[2]} \cdots \geq X_{[n]}$. Recall that the classical estimator for
$\calp(X)$ given the samples $\{X_{i}\}_{i=1}^{n}$ is
\begin{align*}
\calpn(X) = \frac{1}{n(1-\alpha)}\sum_{i=1}^{n}X_{i}
\mathbbm{I}\{X_{i} \geq \valpn(X)\},
\end{align*}
where $\valpn(X) = X_{[\floor{n(1-\alpha)}]}$ is an estimator for $\valp(X).$

We begin by proving a concentration inequality for $\calpn(X)$ for the
special case when $X$ is bounded.
\begin{theorem}
\label{cvar-bounded-theorem}
For $b > 0,$ suppose that $X$ satisfies supp($X$) $\subseteq
[-b,b]$. Then for any $\varepsilon \geq 0$,
%\begin{subequations}
\begin{align*}
%&\pbb(\calpn(X) \leq \calp(X) - \varepsilon) \leq 3 \text{exp} \Big(-n(1-\alpha) \frac{(\varepsilon/b)^{2}}{9 + 1.6\varepsilon/b} \Big) \label{eq:cvar-bounded-theorem-a}\\
%&\pbb(\calpn(X) \geq \calp(X) + \varepsilon) \leq 3 \text{exp} \bigg(-n(1-\alpha) \frac{(\varepsilon/b)^{2}}{10 + 1.4\varepsilon/b} \bigg) \label{eq:cvar-bounded-theorem-b} \label{eq:cvar-bounded-theorem-b}
\prob{|\calpn(X) -  \calp(X)| \geq \varepsilon} \leq 6 \text{exp} \bigg(-n(1-\alpha) \frac{(\varepsilon/b)^{2}}{10 + 1.6\varepsilon/b} \bigg). 
\end{align*}
%\end{subequations}
\end{theorem}

Theorem~\ref{cvar-bounded-theorem} is a refinement of the CVaR
concentration inequality for bounded distributions in \cite{wang2010}. The
proof can be found in Appendix~\ref{cvar-bounded-proof}. 

We now use Theorem~\ref{cvar-bounded-theorem} to develop a CVaR
concentration inequality for unbounded (potentially heavy-tailed)
distributions. In particular, our concentration inequality applies to
the following truncation-based estimator. For $b > 0,$ define
$$X^{(b)}_{i} = \min(\max(-b,X_i),b).$$ Note that $X_i^{(b)}$ is simply
the projection of $X_i$ onto the interval $[-b,b].$ Let
$\{X^{(b)}_{[i]}\}_{i=1}^{n}$ denote the order statistics of truncated
samples $\{X^{(b)}_{i}\}_{i=1}^{n}.$ Our estimator $\cvarte(X)$ for
$\calp(X)$ is simply the empirical CVaR estimator for $X^{(b)} :=
\min(\max(-b,X),b),$ i.e.,
\begin{align*}
  \cvarte(X) = \calpn(X^{(b)}) = \frac{1}{n(1-\alpha)}\sum_{i=1}^{n}
  X^{(b)}_{i} \mathbbm{I}\{X^{(b)}_{i} \geq \valpn(X^{(b)})\},
\end{align*}
where $\valpn(X^{(b)}) = X^{(b)}_{[\floor{n(1-\alpha)}]}.$

Note that the nature of truncation performed here is different from
that in the conventional truncation-based mean estimators (see, for
example, \cite{bubeck2013}), where samples with an absolute value
greater than $b$ are set to zero. In contrast our estimator projects
these samples to the interval $[-b,b].$ This difference plays an
important role in establishing the concentration properties of the
estimator.

We are now ready to state our main result, which shows that the
truncation-based estimator $\cvarte(X)$ works well when the truncation
parameter $b$ is large enough.

\begin{theorem}
\label{cvar-ht-theorem}
Suppose that $\{X_{i}\}_{i=1}^{n}$ are i.i.d. samples distributed as
$X,$ where $X$ satisfies conditions {\bf C1} and {\bf C2}. Given
$\Delta > 0,$ 
\begin{equation}
\label{eq:cvar-ht-bound}
\prob{|\calp(X) - \cvarte(X)| \geq \Delta} \leq 6 \text{exp} 
\bigg(-n(1-\alpha)\frac{\Delta^2}{48 b^2}\bigg)
\end{equation} 
\begin{equation}
  \label{eq:cvar-ht-truncation_bound}
  \text{ for } b > \max\left(\frac{\Delta}{2}, |\valp(X)|,
    \left[\frac{2B}{\Delta(1-\alpha)}\right]^{\frac{1}{p-1}} \right).
\end{equation}
\end{theorem}

The proof of Theorem~\ref{cvar-ht-theorem} can be found in
Appendix~\ref{cvar-ht-proof}. The key feature of truncation-based
estimators like the one proposed here for the CVaR is that they enable
a parameterized bias-variance tradeoff. While the truncation of the
data itself adds a bias to the estimator, the boundedness of the
(truncated) data limits the variability of the estimator. Indeed, the
condition that
$b > \left[\frac{2B}{\Delta(1-\alpha)}\right]^{\frac{1}{p-1}}$ in the
statement of Theorem~\ref{cvar-ht-theorem} ensures that the estimator
bias induced by the truncation is at most $\Delta/2.$

In practice, one might not know the values of $\valp(X),$ $B,$ $p$ or
even $\Delta$ (as is the case in MAB problems), so ensuring that the
lower bound on $b$ is satisfied is problematic.\footnote{We note here
  that $|\valp(X)|$ can be upper bounded in terms of $p$ and $B$ as
  follows:
  $|\valp(X)| \leq \Big(\frac{B}{\text{min}(\alpha,\beta) }
  \Big)^{\frac{1}{p}}$ (see Appendix~\ref{var-mag-bound}).
  %Thus, one can alternatively write a (weaker)
  %threshold on $b$ only in terms of $B,$ $p$ and $\Delta.$
  Thus,
  $b > \Big(\frac{B}{\text{min}(\alpha,\beta) } \Big)^{\frac{1}{p}}$
  implies $b> |\valp(X)|.$} The natural strategy to follow then is to
set the truncation parameter as an increasing function of the number
of data samples $n,$ which ensures that
\eqref{eq:cvar-ht-truncation_bound} holds for large enough $n.$
Moreover, it is clear from \eqref{eq:cvar-ht-bound} that for the
estimation error to (be guaranteed to) decay with $n,$ $b^2$ can grow
at most linearly in $n.$ Indeed, for our bandit algorithms, we set
$b = n^{q},$ where $q \in (0,1/2).$

Finally, it is tempting to set $b$ in a \emph{data-driven} manner,
i.e., to estimate the VaR, moment bounds and so on from the data, and
set $b$ large enough so that \eqref{eq:cvar-ht-truncation_bound} holds
with high probability. The issue however is that $b$ then becomes a
(data-dependent) random variable, and proving concentration results
with such data-dependent truncation is much harder.
% ; see the discussion in Section~\ref{sec:discuss}.

%------------------------------------------------------------------------

\section{Risk-aware, distribution oblivious algorithms for MAB}
\label{sec:algorithms}

In this section, we formulate the problem of best arm identification
in a risk-aware fashion, propose algorithms, and prove performance
guarantees for these algorithms.

Consider a multi-armed bandit problem with $K$ arms, labeled
$1,2,\cdots,K.$ The loss (or cost) associated with arm~$i$ is
distributed as $X(i),$ where it is assumed that there exists $p > 1$
and $B < \infty$ such that $\Exp{|X(i)|^p} < B$ for all
$i.$\footnote{We pose the problem as (risk-aware) loss minimization,
  which is of course equivalent to (risk-aware) reward maximization.}
Each time an arm~$i$ is pulled, an independent sample distributed as
$X(i)$ is observed. Given a fixed budget of $T$ arm pulls in total,
our goal is to identify the arm that minimizes
$\xi_{1} \Exp{X(i)} + \xi_{2}\calp(X(i)),$ where $\xi_1$ and $\xi_2$
are positive (and given) weights. $(\xi_1,\xi_2) = (1,0)$ corresponds
to the classical mean minimization problem
\citep{audibert2010,yu2018}, whereas $(\xi_1,\xi_2) = (0,1)$
corresponds to a pure CVaR minimization
\citep{galichet2013,kolla2019b}. Optimization of a linear combination
of the mean and CVaR has been considered before in the context of
portfolio optimization in the finance community, but not, to the best
of our knowledge, in the MAB framework. The performance metric we
consider is the probability of incorrect arm identification. For
simplicity, we assume that the distributions of the arms satisfy
condition \textbf{C1}.

We also assume that there is a unique optimal arm. This is purely for
simplicity in expressing our performance guarantees; it is
straighforward to extend these to the setting where there are multiple
optimal arms. Let the ordered suboptimality gaps for the metric
$\xi_{1}\Exp{X(\cdot)} + \xi_{2}\calp(X(\cdot))$ be denoted by
$\Delta[2],\cdots,\Delta[K];$ here,
$0<\Delta[2] \leq \cdots \leq \Delta[K]$.
% The square brackets will be used to denote the ordering of arms.

Finally, recall that we consider an entirely distribution oblivious
environment. In other words, the algorithm does not have any prior
information about the arm distributions, including the values of $p$
and $B.$ This is in contrast with the most of the literature on MAB
problems, where information about the support of the arm
distributions, bounds on their moments and/or sub-optimality gaps are
baked into the algorithms.\footnote{While the algorithms we propose
  are distribution oblivious, their performance guarantees will of
  course depend on the arm distributions.}
% JK: No need to citation on this line, IMO.

\subsection{Algorithms}

We estimate the performance of each arm as follows. Suppose that
arm~$i$ has been pulled $n$ times, and we observe samples
$X_1^i,X_2^i,\cdots,X_n^i.$ We use the following truncated empirical
estimator (see \cite{bickel1965,bubeck2013}) for the mean value
associated with the arm:
\begin{equation*}
  \tea_{n}(i) := \frac{\sum_{j = 1}^{n} X_{j}^i \Ind{|X_{j}^i| \leq b_m(n)} }{n},
\end{equation*} 
where $b_m(n) = n^{q_m}$ for $q_m \in (0,1).$ Note that we are growing
the truncation parameter $b_m$ sub-linearly in $n.$
Our estimator for the CVaR associated with arm~$i$ is the one
developed in Section~\ref{sec:cvar_concentration}, i.e.,
\begin{equation*}
  \cvargen = \hat{c}_{n,\alpha}^{(b_c(n))},
\end{equation*}
where $b_c(n) = n^{q_c}$ for $q_c \in (0,1/2).$

Our algorithms are of \emph{successive rejects} type
\citep{audibert2010}. They are parameterized by non-negative integers
$n_1 \leq n_2 \leq \cdots \leq n_{K-1}$ satisfying
$\sum_{i=1}^{K-2} n_i + 2 n_{K-1} \leq T.$ The algorithm proceeds in
$K-1$ phases, with one arm being rejected from further consideration
at the end of each phase. In phase~$i,$ the $K-1+i$ arms under
consideration are pulled $n_i-n_{i-1}$ times, after which the arm with
the worst (estimated) performance is rejected. This is formally
expressed in Algorithm~\ref{alg:gsr}.
The classical successive rejects algorithm in \cite{audibert2010} used
$n_k \propto \frac{T-K}{K+1-k}.$ Another special case is \emph{uniform
  exploration,} where $n_1 = n_2 = \cdots n_{K-1} = \floor{T/K}.$ As
the name suggests, under uniform exploration, all arms are pulled an
equal number of times, after which the arm with the best estimate is
selected.

\begin{algorithm}
  \caption{Generalized successive rejects algorithm}
  \label{alg:gsr}
\begin{algorithmic}
  \Procedure{GSR}{$T, K, \{n_{1},\cdots,n_{K-1}\}$}
  \State $A_{1} \gets \{1,\cdots,K\}$
  \State $n_{0} \gets 0$
  \For{$k = 1 \text{ to }K-1$}
  \State $\text{For each } i \in A_{k}\text{, select arm } i \text{ for } n_{k}-n_{k-1} \text{ rounds.}$
  \State $\text{Let } A_{k+1} = A_{k} \setminus \text{arg max}_{i \in A_{k}} \xi_{1} \meantea_{n_{k}}(i) + \xi_{2} \calptk(i)$
  \EndFor
  \State $\text{Output unique element of } A_{K}$
  \EndProcedure
\end{algorithmic}
\end{algorithm} 

\subsection{Performance evaluation}

We now state upper bounds on the probability of incorrect arm
identification under the successive rejects and uniform exploration
algorithms. However, our bounding techniques easily extend to the
complete class of generalized successive rejects algorithms described
in Algorithm~\ref{alg:gsr}.

\begin{theorem}
  \label{ue-prob-of-error}
  Suppose that the arm distributions satisfy the conditions
  \textbf{C1} and \textbf{C2}.
  
  Under the uniform exploration algorithm, the probability of
  incorrect arm identification $p_e$ is bounded as
  \begin{equation*}
    p_{e} \leq 2K\text{exp}\Big(-(T/K)^{1-q_m}\frac{\Delta[2]}{16\xi_{1}}\Big) 
    + 6K \text{exp} \Big(-(T/K)^{1-2q_c} \frac{\beta \Delta[2]^{2}}{768 \xi_{2}^{2}}\Big) 
  \end{equation*}
  for $T> Kn^{*},$ where
  \begin{equation*}
    n^{*} = \max \bigg(\Big(\frac{12\xi_{1}B}{\Delta[2]}\Big)^{\frac{1}{q_m\min(p-1,1)}}, 
    \Big(\frac{8\xi_{2}B}{\beta\Delta[2]}\Big)^{\frac{1}{q_c(p-1)}}, 
    \Big(\frac{B}{\text{min}(\alpha,\beta) } \Big)^{\frac{1}{q_c p}},
    \Big(\frac{\Delta[2]}{8\xi_{2}} \Big)^{\frac{1}{q_c}} \bigg). 
  \end{equation*}
\end{theorem}

The proof of Theorem~\ref{ue-prob-of-error} can be found in
Appendix~\ref{gsr-theorem-proof}. Here, we highlight the main takeaways from this
result.

First, note that the probability of error (incorrect arm
identification) decays to zero as $T \ra \infty.$ However, \emph{the
  decay is slower than exponential in $T;$} taking $q_m = q,$
$q_c = q/2$ for $q \in (0,1),$ the probability of error is
$O(\mathrm{exp}(-\gamma T^{1-q}))$ for a positive constant $\gamma.$
This slower-than-exponential bound is a consequence of the
distribution obliviousness of the algorithm. In technical terms, this
results from having to set the truncation parameters $b_m$ and $b_c$
for each arm as increasing functions of the horizon $T.$
% which increases the variability of the truncated estimators.
Indeed, as we show in Section~\ref{sec:nonoblivious}, if $B,$ $p,$ and
$\Delta[2]$ are known to the algorithm (as is often assumed in the
literature), then it is possible to achieve an exponential decay of
the probability of error with $T;$ in this case, it is possible to
simply set the truncation parameters as static constants (that do not
depend on~$T$).

Second, our upper bounds only hold when $T$ is larger than a certain
threshold. This is again a consequence of distribution
obliviousness---the concentration inequalities on our truncated
estimators are only valid when the truncation interval is wide
enough. This is required in order to limit the bias of these
estimators. As a consequence, our performance guarantees only kick in
once the horizon length is large enough to ensure that this condition
is met. As expected, in the non-oblivious setting, this limitation
does not arise, since the truncation parameters can be statically set
to be large enough to limit the bias (see
Section~\ref{sec:nonoblivious}).

Third, there is a natural tension between the bound for the
probability of error and the threshold on $T$ beyond which they are
applicable, with respect to the choice of truncation parameters $q_m$
and $q_c.$ In particular, the the upper bound on $p_e$ decays fastest
with respect to $T$ when $q_m,q_c \approx 0.$ However, choosing
$q_m,q_c$ to be small would make the threshold on the horizon to be
large, since the bias of our estimators would decay slower with
respect to $T.$ Intuitively, smaller values of $q_m,q_c$ limit the
variance of our estimators (which is reflected in the bound for $p_e$)
at the expense of a greater bias (which is reflected in the threshold
on $T$), whereas larger values $q_m,q_c$ limit the bias at the expense
of increased variance. We comment on the best choice of these
parameters as suggested by numerical experimentation in
Section~\ref{sec:experiments}.

Finally, we note that the bound on the probability of error in
Theorem~\ref{sr-prob-of-error} is stronger than the power law bound
corresponding to the distribution oblivious algorithm for the mean
metric analysed in \cite{yu2018}. The latter uses the standard
(non-truncated) empirical mean estimator, which has weaker
concentration properties compared to the truncated empirical mean
estimator used here.

Next, we consider the successive rejects algorithm. Let
$\overline{\log}(K) := 1/2 + \sum_{i=2}^{K} 1/i$.

\begin{theorem}
    \label{sr-prob-of-error}
    Let the arms satisfy the conditions \textbf{C1} and
    \textbf{C2}.
    The probability of incorrect arm identification for the successive
    rejects algorithm is bounded as follows.
    \begin{align*}
      p_{e} &\leq \sum_{i=2}^{K} (K+1-i) 2\text{exp}\bigg(-\frac{1}{16\xi_{1}}\Big(\frac{T-K}{\overline{\log}(K)}\Big)^{1-q_m} \frac{\Delta[i]}{i^{1-q_m}}\bigg) \\
            & \qquad +\sum_{i=2}^{K} (K+1-i) 6\text{exp} \bigg(-\frac{\beta}{768 \xi_{2}^{2}}\Big(\frac{T-K}{\overline{\log}(K)}\Big)^{1-2q_c} \frac{\Delta[i]^{2}}{i^{1-2q_c}} \bigg)
\end{align*}
for $T> K+K\overline{\log}(K)n^{*},$ where
\begin{equation*}
    n^{*} = \max \bigg(\Big(\frac{12\xi_{1}B}{\Delta[2]}\Big)^{\frac{1}{q_m\min(p-1,1)}}, 
    \Big(\frac{8\xi_{2}B}{\beta\Delta[2]}\Big)^{\frac{1}{q_c(p-1)}}, 
    \Big(\frac{B}{\text{min}(\alpha,\beta) } \Big)^{\frac{1}{q_c p}},
    \Big(\frac{\Delta[2]}{8\xi_{2}} \Big)^{\frac{1}{q_c}} \bigg).   
  \end{equation*}
\end{theorem}
Structurally, our results for the successive rejects algorithm are
similar to those for uniform exploration. Indeed, taking $q_m = q,$
$q_c = q/2$ for $q \in (0,1),$ the probability of error remains
$O(\mathrm{exp}(-\gamma T^{1-q}))$ for a different positive constant
$\gamma.$ So our conclusions from Theorem~\ref{ue-prob-of-error},
including the bias-variance tradeoff in setting the truncation
parameters $q_m$ and $q_c,$ apply to Theorem~\ref{sr-prob-of-error} as
well. 
Intuitively, one would expect the successive rejects algorithm
to perform better when the arms are well separated, whereas uniform
exploration would work well when all sub-optimal arms are nearly
identically separated from the optimal arm.

%%%%%%%%%%%%%%%%%%%%%%%%%%%%%%%%%%%%%%%%%%%% 
%%%%%%%%%%%%%%%%%%%%%%%%%%%%%%%%%%%%%%%%%%%% 
\subsection{The non-oblivious setting}
\label{sec:nonoblivious}

Finally, we consider the non-oblivious setting, where the algorithm
knows $p,$ $B$ and $\Delta[2]$ (or a lower bound on $\Delta[2]$). This
is the setting that is effectively considered in the bulk of the
literature on MAB algorithms. In this case, we show that it is
possible to set the algorithm parameters (specifically, the truncation
parameters) so that we achieve an exponential decay of the probability
of error with~$T.$ Moreover, unlike our results for the distribution
oblivious case, there is no lower bound on $T$ beyond which the bounds
on the probability of error apply.

In particular, we set truncation threshold for the mean estimator as
$b_m = \Big(\frac{12B\xi_{1}}{\Delta[2]}\Big)^{\frac{1}{\min(1,p-1)}}$
and the truncation threshold for the CVaR estimator as
$b_c = \max
\bigg(\Big(\frac{8\xi_{2}B}{\beta\Delta[2]}\Big)^{\frac{1}{p-1}},
\Big(\frac{B}{\text{min}(\alpha,\beta) }
\Big)^{\frac{1}{p}}\bigg)$. It can be shown that this would ensure an
exponentially decaying (in $T$) probability of error for uniform
exploration as well as successive rejects (see
Appendix~\ref{nonoblivious-proofs}).

In conclusion, the results in this section show that one can indeed
devise algorithms for risk-aware best arm identification in an
entirely distribution oblivious manner. However, the performance
guarantees we obtain are not as strong as those that can be obtained
for non-oblivious algorithms; this is of course what one would
expect. In the next section, we evaluate the performance gap between
oblivious and non-oblivious algorithms via numerical experiments.

%------------------------------------------------------------------------

\section{Numerical Experiments}
\label{sec:experiments}

In this section, we evaluate the performance of the proposed
algorithms via simulations, by making the comparison with (more
conventional) non-oblivious algorithms. Due to space constraints, we
restrict ourselves to successive rejects (SR) algorithms, and two
specific objectives: (i) mean minimization, i.e.,
$(\xi_1,\xi_2) = (0,1),$ and (ii) CVaR minimization, i.e.,
$(\xi_1,\xi_2) = (1,0).$ In each of the experiment below, the
probability of error is computed by averaging over 50000 runs at each
sampled $T$.

We consider the following MAB problem instance: There are 10 arms, the
first having mean loss 0.9, and the remaining having mean loss 1. The
first 5 arms have a (heavy-tailed) Pareto loss distribution with shape
parameter 3, and the last 5 arms have an exponential loss
distribution. The confidence level $\alpha$ is set to 0.95. In this
case, Arm~1 is optimal for the mean as well as the CVaR metric.

For the mean minimization problem, our results are presented in
Figure~\ref{fig:mean_expt1}. We compare the probability of error for
the proposed distribution oblivious algorithm (taking $q_m = 0.75$)
with that corresponding to the non-oblivious truncation based SR
algorithm from \cite{yu2018}, which uses the information $p=2, B=2.0,$
and $\Delta[2] = 0.1.$ For the CVaR minimization problem, our results
are presented in Figure~\ref{fig:cvar_expt1}. Again, we compare the
error probability of the of the proposed oblivious SR algorithm with
that corresponding to a non-oblivious truncation based SR algorithm
with $b_m = (4B/(\Delta[2]\beta))^{1/(p-1)},$ where $p=2, B=2.0,$ and
the minimum CVaR gap $\Delta[2] = 0.25$ (this ensures an exponential
decay in $T$ of the error probability). Note that the performance of
the proposed oblivious algorithms is almost indistinguishable from
that of the non-oblivious counterparts.

\begin{figure}[htp]
    \centering
    \begin{subfigure}[b]{0.45\linewidth}
        \includegraphics[width=\linewidth]{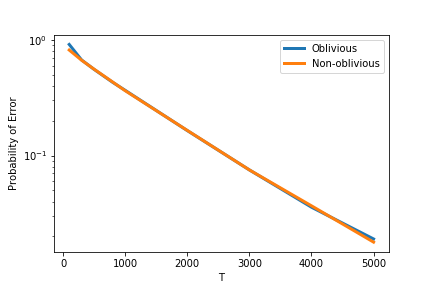}
        \caption{Mean Minimization}
        \label{fig:mean_expt1}
    \end{subfigure}
    \begin{subfigure}[b]{0.45\linewidth}
        \includegraphics[width=\linewidth]{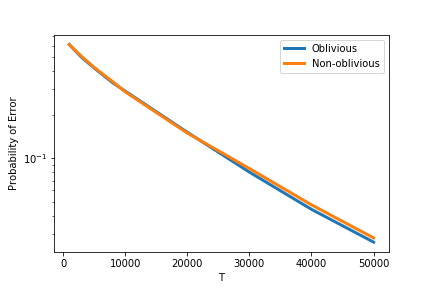}
        \caption{CVaR Minimization}
        \label{fig:cvar_expt1}
    \end{subfigure}
    \label{fig:performance}
    \caption{Performance comparison between oblivious and
      non-oblivious SR algorithms}
\end{figure}

Next, we illustrate an instance where there is a visible performance
hit associated with distribution obliviousness. Consider an MAB
problem with two arms for the mean minimization metric. The first arm
has a Pareto distribution with mean loss of 1.0 and shape parameter
1.9. The second arm is exponentially distributed with a mean loss of
0.9. While the second arm is optimal, our truncation induces a greater
bias (specifically, underestimation) in the mean estimate of the
(heavy-tailed) first arm compared to the second, resulting in poorer
performance when $q_m$ is small. To see this, we compare the
performance of the oblivious SR algorithm to the non-oblivious SR
algorithm of \cite{yu2018} for $q_m = 0.4,$ 0.5, and 0.7 (see
Figure~\ref{fig:projection_growth}). Note that when $q_m$ is small,
the truncation interval grows slowly with $T,$ and the resulting bias
gets reflected in poorer performance compared to the non-oblivious
algorithm. On the other hand, for $q_m = 0.7,$ the truncation interval
grows fast enough to make the performance indistinguishable from the
non-oblivious algorithm.

\begin{figure}[htp]
    \centering
    \begin{subfigure}[b]{0.30\linewidth}
        \includegraphics[width=\linewidth]{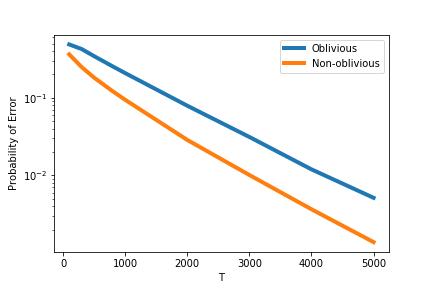}
        \caption{$b_{m} = n^{0.4}$ }
    \end{subfigure}
    \begin{subfigure}[b]{0.30\linewidth}
        \includegraphics[width=\linewidth]{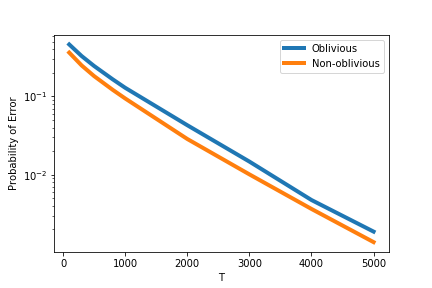}
        \caption{$b_{m} = n^{0.5}$}
    \end{subfigure}
    \begin{subfigure}[b]{0.30\linewidth}
        \includegraphics[width=\linewidth]{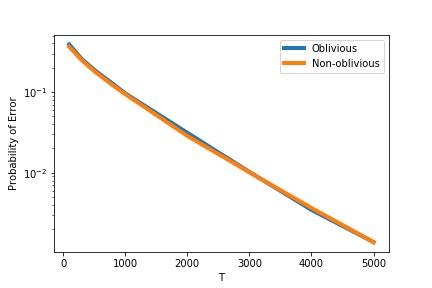}
        \caption{$b_{m} = n^{0.7}$}
    \end{subfigure}
    \caption{Performance of SR for mean minimization with different
      truncation interval growth rates}
    \label{fig:projection_growth}
  \end{figure}

\section{Concluding Remarks}
\label{sec:discuss}
In this paper, we consider the problem of risk-aware best arm
selection in a pure exploration MAB framework. A key feature of our
algorithms is {\it distribution obliviousness}; the algorithms have no
prior knowledge about the arm distributions. This is in contrast with
most algorithms in the literature for MAB problems, which assume prior
knowledge of the support, moment bounds, or bounds on the
sub-optimality gaps. The proposed algorithms come with analytical
performance guarantees, and also seem to perform well in practice.

This paper motivates future work along several directions. First, our
numerical experiments suggest that our upper bounds on the probability
of error for the distribution oblivious algorithms are rather
loose. Tigher performance bounds, which would in turn require tigher
concentration bounds for truncation-based estimators, are worth
exploring. More importantly, fundamental lower bounds on the
performance of any algorithm need to be devised for the distribution
oblivious setting. Currently available lower bounds (see
\cite{audibert2010}) on the error probability for best arm
identification do not take into account the information available to
the algorithm, and only capture risk-neutral arm selection. Finally,
it is also interesting to explore distribution oblivious algorithms in
the regret minimization framework, as well as the PAC framework.

\bibliography{references}

\newpage
\appendix

\section{CVaR Concentration for Bounded Random Variables (Proof of Theorem~\ref{cvar-bounded-theorem})}
\label{cvar-bounded-proof}
We state two concentration inequalities that will be used repeatedly
in the proof of Theorem~\ref{cvar-bounded-theorem}.

\textbf{Bernstein's Inequality}

Let $X_{i}$ be IID samples of a random variable $X$ with mean $\mu$. If $|X| \leq b$ almost surely, then for any $\epsilon > 0$,
\begin{equation*}
    \pbb\bigg(\bigg|\frac{\sum_{i=1}^{n}X_{i}}{n} - \mu\bigg| > \epsilon \bigg) \leq 2\text{exp}\bigg(-\frac{n \epsilon^{2}}{2\mathbb{E}[X^{2}]+2b\epsilon/3} \bigg)
\end{equation*}

\textbf{Chernoff Bound for Bernoulli Experiment}

Let $X_{1},...,X_{n}$ be independent Bernoulli experiments, $\pbb(X_{i} = 1) = p_{i}$. Set $X = \sum_{i=1}^{n}X_{i}$, $\mu = \mathbb{E}[X]$. Then for every $0<\delta<1$, 
\begin{align*}
	P(X \geq (1+\delta)\mu) \leq \text{exp}(-\frac{\mu\delta^{2}}{3}) \\
	P(X \leq (1-\delta)\mu) \leq \text{exp}(-\frac{\mu\delta^{2}}{2})
\end{align*}

Theorem~\ref{cvar-bounded-theorem} follows from the following
statement.  Let $X$ be any random variable with supp($X$)
$\subseteq [-b,b]$. Then for any $\varepsilon \geq 0$,
\begin{subequations}
\begin{align}
    &\pbb(\calpn(X) \leq \calp(X) - \varepsilon) \leq 3 \text{exp} \Big(-(1-\alpha) n \frac{(\varepsilon/b)^{2}}{9 + 1.6\varepsilon/b} \Big) \label{eq:cvar-bounded-theorem-a}\\
    &\pbb(\calpn(X) \geq \calp(X) + \varepsilon) \leq 3 \text{exp} \bigg(-n(1-\alpha) \frac{(\varepsilon/b)^{2}}{10 + 1.4\varepsilon/b} \bigg) \label{eq:cvar-bounded-theorem-b}
\end{align}
\end{subequations}

\subsection{Proof of \ref{eq:cvar-bounded-theorem-a}}
\label{cvar-bounded-proof-a}

We're going to use the following lemma from \cite{wang2010}.

\begin{lemma}
    Let $X_{[i]}$ be the decreasing order statistics of $X_{i}$; then $f(k) = \frac{1}{k} \sum_{i=1}^{k}X_{[i]}, ~1 \leq k \leq n$, is decreasing and the following two inequalities hold:
    \begin{subequations}
    \begin{align}
        \frac{1}{n\beta} \sum_{i=1}^{\fnb} X_{[i]} &\leq \calpn(X) \leq \frac{1}{n\beta} \sum_{i=1}^{\cnb} X_{[i]} \label{eq:wang-a} \\
        f(\cnb) &\leq \calpn(X) \leq f(\fnb) \label{eq:wang-b}
    \end{align}
    \end{subequations}  
\end{lemma}

\textbf{1. $\varepsilon > 2b$}
\begin{equation*}
    \pbb(\calpn(X) \leq \calp(X) - \varepsilon) = 0
\end{equation*}
Both $\calp(X) \in [-b,b]$ and $\calpn(X) \in [-b,b]$. Therefore, the difference can't be larger than $2b$.

\vspace{5mm}
\textbf{2. $\varepsilon \in [0,2b]$}

We'll condition the probability above on a random variable $\knb$ which is defined as $\knb$ = max$\{i: X_{[i]} \in [\valp(X), b]\}$. Note that $\valp(X)$ is a constant such that the probability of a $X$ being greater than $\valp(X)$ is $\beta$. Also observe that $\pbb(\knb = k) = \pbb(k$ from $\{X_{i}\}_{i=1}^{n}$ have values in $[\valp(X), b])$. Using the above two statements one can easily see that $\knb$ follows a binomial distribution with parameters $n$ and $\beta$.

Consider $k$ I.I.D. random variables $\{\Tilde{X}_{i}\}_{i=1}^{k}$ which are distributed according to $\pbb(X \in \cdot ~| X \in [\valp(X), b])$.  By conditioning on $\knb = k$, one can observe using symmetry that $\frac{1}{k}\sum_{i=1}^{k} X_{[i]}$ and $\frac{1}{k}\sum_{i=1}^{k}\Tilde{X}_{i}$ have the same distribution. We'll next bound the probability  $\pbb(\calpn(X) \leq \calp(X) - \varepsilon | \knb = k)$ for different values of $k$. Now,
\begin{align*}
    \pbb(\calpn(X) \leq \calp(X) - \varepsilon) &= \sum_{k=0}^{n}\pbb(\knb = k)\pbb(A) \\
    & \leq \underbrace{\sum_{k=0}^{\fnb}\pbb(\knb = k)\pbb(A)}_{I_{2}}  + \underbrace{\sum_{k=\cnb}^{n}\pbb(\knb = k)\pbb(A)}_{I_{1}}
\end{align*}
where $\pbb(A) = \pbb(\calpn(X) \leq \calp(X) - \varepsilon | \knb = k)$. 

\vspace{3mm}
\textbf{Bounding $I_{1}$}

Note that $k \geq \cnb$. We'll begin by bounding $P(A)$.
\begin{align*}
    &\pbb(\calpn(X) \leq \calp(X) - \varepsilon | \knb = k) \\
    \leq~ &\pbb \bigg( \frac{1}{\cnb} \sum_{i=1}^{\cnb} X_{[i]}\leq \calp(X) - \varepsilon | \knb = k \bigg) ~(\text{using \ref{eq:wang-b}}) \\
    \leq~ &\pbb \bigg(\frac{1}{k} \sum_{i=1}^{k} X_{[i]} \leq \calp(X) - \varepsilon | \knb = k \bigg) ~(\because f(\cdot) \text{ is decreasing}) \\
     =~ &\pbb \bigg(\frac{1}{k} \sum_{i=1}^{k} \Tilde{X}_{i} \leq \calp(X) - \varepsilon \bigg) \\
    \leq~ &\text{exp}\bigg(-\frac{k\varepsilon^{2}}{2\mathbb{E}[\Tilde{X}^{2}] + 2b\varepsilon/3} \bigg) ~(\text{using Bernstein's inequality})
\end{align*}
Supp$(\Tilde{X}) \in [\valp(X), b]$. In worst case, $\valp(X) = -b$. Therefore, Supp$(\Tilde{X}) \in [-b, b]$ and $\mathbb{E}[\Tilde{X}^{2}] \leq b^{2}$.  Hence,
\begin{equation*}
    \pbb(\calpn(X) \leq \calp(X) - \varepsilon | \knb = k) \leq \text{exp}\bigg(-\frac{k\varepsilon^{2}}{2b^{2} + 2b\varepsilon/3} \bigg)
\end{equation*}
Hence, we have the following:
\begin{align*}
    I_{1} &= \sum_{k=\cnb}^{n} \binom{n}{k} \beta^{k} (1-\beta)^{n-k}\pbb(\calpn(X) \leq \calp(X) - \varepsilon | \knb = k) \nonumber \\
    &\leq \sum_{k=\cnb}^{n} \binom{n}{k} \bigg( \beta\text{exp}\Big(-\frac{\varepsilon^{2}}{2b^{2} + 2b\varepsilon/3} \Big)\bigg)^{k}(1-\beta)^{n-k} \nonumber\\
    &\leq \bigg(1-\beta+\beta\text{exp}\Big(-\frac{\varepsilon^{2}}{2b^{2} + 2b\varepsilon/3}\Big) \bigg)^{n} \nonumber \\
    &\leq \text{exp}\Bigg(-\beta n \bigg( 1 - \text{exp}\Big(-\frac{\varepsilon^{2}}{2b^{2} + 2b\varepsilon/3}\Big) \bigg)  \Bigg) ~(\because e^{x} \geq 1+x ~\forall x \in \mathbb{R}) % \label{eq:7}
\end{align*}

Now, let's bound $1 - \text{exp}\Big(-\frac{\varepsilon^{2}}{2b^{2} + 2b\varepsilon/3}\Big)$. We know that $1 - e^{-x} \geq x - x^{2}/2 = x(1 - x/2)$. One can easily verify that $\frac{\varepsilon^{2}}{2b^{2} + 2b\varepsilon/3}$ is an increasing function of $\varepsilon$ if $\varepsilon \geq 0$. Putting $\varepsilon = 2b$, we get, $1 - \frac{1}{2}\frac{\varepsilon^{2}}{2b^{2} + 2b\varepsilon/3} \geq 1 - \frac{3}{5} = \frac{2}{5}$. Hence, $1 - \text{exp}\Big(-\frac{\varepsilon^{2}}{2b^{2} + 2b\varepsilon/3}\Big) \geq \frac{2}{5}\frac{\varepsilon^{2}}{2b^{2} + 2b\varepsilon/3}$. Therefore,
\begin{align*}
    I_{1} \leq \text{exp}\bigg(-\beta n \Big(\frac{(\varepsilon/b)^{2}}{5 + 5\varepsilon/(3b)} \Big) \bigg)
\end{align*}
\vspace{3mm}

\textbf{Bounding $I_{2}$} 

Note that $k \leq \fnb$. We'll again start by bounding $\pbb(A)$.

\begin{align*}
    \pbb(\calpn(X) \leq \calp(X) - \varepsilon | \knb = k) &\leq \pbb \Big( \frac{1}{n\beta} \sum_{i=1}^{\fnb} X_{[i]} \leq \calp(X) - \varepsilon \Big| \knb = k \Big) ~(\text{Using \ref{eq:wang-a}}) \\
    & \leq \pbb \Big( \frac{1}{k} \sum_{i=1}^{k} X_{[i]} \leq \frac{n\beta}{k}(\calp(X) - \varepsilon) \Big| \knb = k \Big) ~(\because~k\leq \fnb) \\
    & \leq \pbb \bigg( \frac{1}{k} \sum_{i=1}^{k} X_{[i]} \leq \calp(X) + \Big(\frac{n\beta}{k} - 1\Big)b - \frac{n\beta\varepsilon}{k} \Big| \knb = k \bigg) ~(\because \calp(X) \leq b)
\end{align*}

\textbf{Case 1} $\varepsilon \in [b,2b]$

Let $\varepsilon_{1}(k) = \frac{n\beta\varepsilon}{k} + \Big(1 - \frac{n\beta}{k}\Big)b = b \bigg(1 + \Big(\frac{\varepsilon}{b} - 1 \Big)\frac{n\beta}{k} \bigg)$. Note that $\varepsilon_{1}(k) > 0$ for all $k$ as $\varepsilon \geq b$. Also note that $\varepsilon_{1}(k)$ decreases as $k$ increases. As $k \leq n\beta$, $\varepsilon_{1}(k) \geq \varepsilon$.

\begin{align*}
    \pbb \Big(\frac{1}{k} \sum_{i=1}^{k} X_{[i]} \leq \calp(X) - \varepsilon_{1}(k) | \knb = k \Big) &= \pbb \Big(\frac{1}{k} \sum_{i=1}^{k} \Tilde{X}_{i} \leq \calp(X) - \varepsilon_{1}(k) \Big) \\
    & \leq \text{exp}\Big( -\frac{k\varepsilon_{1}^{2}(k)}{2b^{2}+2b\varepsilon_{1}(k)/3} \Big) \\
    & \overset{\mathclap{\text{(a)}}}{\leq} \text{exp}\Big( -\frac{k\varepsilon^{2}}{2b^{2}+2b\varepsilon/3} \Big)
\end{align*}
(a) above follows because $\frac{\varepsilon_{1}^{2}(k)}{2b^{2}+2b\varepsilon_{1}(k)/3}$ is an increasing function of $\varepsilon_{1}(k)$ and $\varepsilon_{1}(k) \geq \varepsilon$.  

Using steps similar to that for bounding $I_{1}$, we have:
\begin{align*}
    I_{2} \leq \text{exp}\bigg(-\beta n \Big(\frac{(\varepsilon/b)^{2}}{5 + 5\varepsilon/(3b)} \Big) \bigg); \quad \varepsilon \in [b,2b]
\end{align*}

\vspace{3mm}
\textbf{2.} $\varepsilon \in [0,b)$

Here, $\varepsilon_{1}(k) = \frac{n\beta\varepsilon}{k} - \Big(\frac{n\beta}{k}-1 \Big)b = b \bigg(1 - \Big(1- \frac{\varepsilon}{b}  \Big)\frac{n\beta}{k} \bigg)$. Note that $\varepsilon_{1}(k)>0$ iff $k>n\beta(1-\frac{\varepsilon}{b})$.

\textbf{Case 2.1} If $\varepsilon$ is very small such that $\fnb \leq n\beta\Big(1-\frac{\varepsilon}{b}\Big)$, then $\varepsilon_{1}(k)\leq 0$. Let's bound $I_{2}$ for this case:
\begin{align*}
    I_{2} &\leq \sum_{k=0}^{\fnb} \pbb(\knb = k) \\
    &= \pbb(\knb \leq \fnb) \\
    &\leq \pbb(\knb \leq n\beta(1-\varepsilon/b)) \\
    & \leq \text{exp} \Big(-n\beta \frac{\varepsilon^{2}}{2b^{2}} \Big) \qquad (\text{Chernoff on }\knb)
\end{align*}

\textbf{Case 2.2} $n\beta(1-\varepsilon/b) < \fnb$

Choose $k_{\gamma}^{*} = n\beta(1-\gamma \varepsilon/b)$ for some $\gamma \in [0,1]$. Then, $n\beta(1-\varepsilon/b) \leq k_{\gamma}^{*} \leq n\beta$. 

Assume $k^{*}_{\gamma} <\fnb$. The proof can can be easily adapted when $k^{*}_{\gamma} \geq \fnb$. As we will see, the bound on $I_{2}$ is looser when $k^{*}_{\gamma} < \fnb$.

For $k>k_{\gamma}^{*}$, $\varepsilon(k)>0$. As $k$ increases, $\varepsilon_{1}(k)$ also increases. 

Now, we'll bound $\pbb \Big(\frac{1}{k} \sum_{i=1}^{k} X_{[i]} \leq \calp(X) - \varepsilon_{1}(k) \Big)$:
\begin{align*}
   \pbb \Big(\frac{1}{k} \sum_{i=1}^{k} X_{[i]} \leq \calp(X) - \varepsilon_{1}(k) \Big) &= \pbb \Big(\frac{1}{k} \sum_{i=1}^{k} \Tilde{X}_{i} \leq \calp(X) - \varepsilon_{1}(k) \Big) \\ 
   &\leq \begin{cases}
   \text{exp}\Big( -\frac{k\varepsilon_{1}^{2}(k)}{2b^{2}+2b\varepsilon_{1}(k)/3} \Big); \ksg < k \leq \fnb \\
   1 ; \qquad \qquad\qquad\qquad\qquad\qquad k \leq \ksg
   \end{cases} \\
   & \overset{\mathclap{\text{(a)}}}{\leq} \begin{cases}
   \text{exp} \Big(k\frac{(1-\gamma)^{2}\varepsilon^{2}}{2(b-\gamma\varepsilon)^{2}+ 2(1-\gamma)\varepsilon(b-\gamma\varepsilon)/3} \Big); \ksg < k  \leq \fnb \\
   1 ; \qquad \qquad\qquad\qquad\qquad\qquad k \leq \ksg
   \end{cases} \\
   & \overset{\mathclap{\text{(b)}}}{\leq} \begin{cases}
   \text{exp} \Big(- \frac{k(1-\gamma)^{2}(\varepsilon/b)^{2}}{2+2(1-\gamma)\varepsilon/(3b)} \Big); \ksg < k  \leq \fnb \\
   1 ;\qquad \qquad\qquad\qquad\qquad\qquad k \leq \ksg
   \end{cases}
\end{align*}
(a) above follows because $\frac{\varepsilon_{1}^{2}(k)}{2b^{2}+2b\varepsilon_{1}(k)/3}$ is an increasing function of $\varepsilon_{1}(k)$ and $\varepsilon_{1}(k) \geq \frac{b(1-\gamma)\varepsilon}{b - \gamma\varepsilon}$.

(b) above follows because $b-\gamma\varepsilon \leq b$

Now, we'll bound $I_{2}$:

\begin{align*}
    I_{2} &\leq \sum_{k=0}^{\fnb} \pbb(\knb = k) \pbb \Big(\frac{1}{k} \sum_{i=1}^{k} \Tilde{X}_{i} \leq \calp(X) - \varepsilon_{1}(k) \Big) \\
    &\leq \underbrace{\sum_{k=0}^{\floor{\ksg}} \binom{n}{k}\beta^{k}(1-\beta)^{n-k}}_{I_{2,a}} + \underbrace{\sum_{k=\ceil{\ksg}}^{\fnb}\binom{n}{k}\beta^{k}(1-\beta)^{n-k}\text{exp} \Big(- \frac{k(1-\gamma)^{2}(\varepsilon/b)^{2}}{2+2(1-\gamma)\varepsilon/(3b)} \Big)}_{I_{2,b}} 
\end{align*}

Let's bound $I_{2,a}$. This is very similar to \textit{Case 2.1}.
\begin{align*}
    I_{2,a} & = \sum_{k=0}^{\floor{\ksg}} \binom{n}{k}\beta^{k}(1-\beta)^{n-k} \\
    & = \pbb\big(\knb \leq (1 - \gamma\varepsilon/b)n\beta \big) \\
    & \leq \text{exp} \Big(-n\beta \frac{(\gamma\varepsilon)^{2}}{2b^{2}} \Big)
\end{align*}

If $\ceil{k^{*}_{\gamma}} > \fnb$, $I_{2,b}=0$. 

When $\ceil{k^{*}_{\gamma}} \leq \fnb$, let's bound $I_{2,b}$. This is very similar to bounding $I_{1}$.
\begin{align*}
    I_{2,b} & \leq \bigg(1 - \beta \Big( 1 - \text{exp} \Big(- \frac{(1-\gamma)^{2}(\varepsilon/b)^{2}}{2+2(1-\gamma)\varepsilon/(3b)} \Big)  \Big)\bigg)^{n} \\
    & \leq \text{exp} \bigg(-n\beta \Big( 1 - \text{exp} \Big(- \frac{(1-\gamma)^{2}(\varepsilon/b)^{2}}{2+2(1-\gamma)\varepsilon/(3b)} \Big) \Big) \bigg)
\end{align*}

As $1 - \text{exp}(-x) \geq x - x^{2}/2$ for $x\geq0$
\begin{align*}
    1 - \text{exp} \Big(- \frac{(1-\gamma)^{2}(\varepsilon/b)^{2}}{2+2(1-\gamma)\varepsilon/(3b)} \Big) &\geq  \frac{(1-\gamma)^{2}(\varepsilon/b)^{2}}{2+2(1-\gamma)\varepsilon/(3b)} \Big(1 - \frac{1}{2} \frac{(1-\gamma)^{2}(\varepsilon/b)^{2}}{2+2(1-\gamma)\varepsilon/(3b)} \Big)
\end{align*}
Further,
\begin{align*}
   1 - \frac{1}{2} \frac{(1-\gamma)^{2}(\varepsilon/b)^{2}}{2+2(1-\gamma)\varepsilon/(3b)} \geq 1 - \frac{1}{2}\frac{(1-\gamma)^{2}}{2+2(1-\gamma)/3}  \geq \frac{13}{16}
\end{align*}
Hence, irrespective of whether $\ceil{k^{*}_{\gamma}} \leq \fnb$ or  $\ceil{k^{*}_{\gamma}} > \fnb$:
\begin{align*}
    I_{2,b} \leq \text{exp} \bigg( -n\beta \Big(\frac{13}{16} \frac{(1-\gamma)^{2}(\varepsilon/b)^{2}}{2+2(1-\gamma)\varepsilon/(3b)}\Big)  \bigg)
\end{align*}
Now, we can bound $I_{2}$
\begin{align*}
    I_{2} &\leq I_{2,a} + I_{2,b} \\
    &\leq \text{exp} \Big(-n\beta \frac{\gamma^{2}(\varepsilon/b)^{2}}{2} \Big) + \text{exp} \bigg( -n\beta \Big(\frac{13}{16} \frac{(1-\gamma)^{2}(\varepsilon/b)^{2}}{2+2(1-\gamma)\varepsilon/(3b)}\Big)  \bigg)
\end{align*}

Now, $\gamma^{2} = \frac{13}{16}(1-\gamma)^{2}$ if $\gamma = 1/(1+(16/13)^{0.5}) \approx 0.4740$. Put $\gamma = 0.4740$.
\begin{align*}
    I_{2} \leq 2\text{exp}\Big( -n\beta \frac{0.2247 (\varepsilon/b)^{2}}{2 +  0.351\varepsilon/b} \Big) = 2\text{exp} \Big(-n\beta\frac{(\varepsilon/b)^{2}}{8.9+1.561 \varepsilon/b} \Big)
\end{align*}

Comparing this bound of $I_{2}$ with that of \textbf{Case 2.1}, it is not very difficult to see that the above bound is loose.

Comparing this bound of $I_{2}$ with that of \textbf{Case 1}, notice that the $8.9++1.561 \varepsilon/b \geq 8.9$ whereas $5+5\varepsilon/3b \leq 8.34$. Hence, the above bound is the most general.

Finally, let's bound $I$:

\begin{align*}
    I &\leq I_{1} + I_{2} \\
    &\leq \text{exp}\bigg(-\beta n \Big(\frac{(\varepsilon/b)^{2}}{5 + 5\varepsilon/(3b)} \Big) \bigg) + 2\text{exp}\Big( -n\beta \frac{ (\varepsilon/b)^{2}}{8.9 +  1.561\varepsilon/b} \Big) \\
    & \leq 3 \text{exp} \Big(-\beta n \frac{(\varepsilon/b)^{2}}{9 + 1.6\varepsilon/b} \Big)
\end{align*}

\subsection{Proof of \ref{eq:cvar-bounded-theorem-b}}
\label{cvar-bounded-proof-b}
Let's prove the second part of this theorem now which is the inequality \ref{eq:cvar-bounded-theorem-b}.

Again if $\varepsilon \geq 2b$,
\begin{equation*}
    \pbb(\calpn(X) \geq \calp(X) + \varepsilon) = 0
\end{equation*}
Hence, we're interested in the case where $\varepsilon \in [0, 2b)$. We'll again condition on random variable $\knb$. Remember that $\knb$ follows a binomial distribution with parameters $n$ and $\beta$ . 

The random variables $\{\Tilde{X}_{i}\}_{i=1}^{k}$ are distributed according to $\pbb(X \in \cdot ~ | X \in [\valp(X), b])$. By conditioning of $\knb = k$ distributions of $\frac{1}{k}\sum_{i=1}^{k} X_{[i]}$ and $\frac{1}{k}\sum_{i=1}^{k} \Tilde{X}_{i}$ are same by symmetry. The steps are very similar to that for proving \ref{eq:cvar-bounded-theorem-a}.

\begin{align*}
    \pbb(\calpn(X) \geq \calp(X) + \varepsilon) &= \sum_{k=0}^{n} \pbb(\knb = k) \pbb(A) \\
    &\leq \underbrace{\sum_{k=0}^{\fnb}\pbb(\knb = k)\pbb(A)}_{I_{1}}  + \underbrace{\sum_{k=\cnb}^{n}\pbb(\knb = k)\pbb(A)}_{I_{2}}
\end{align*}

where $\pbb(A) = \pbb(\calpn(X) \geq \calp(X) + \varepsilon | \knb = k)$.
Notice that $I_{1}$ and $I_{2}$ got interchanged from \ref{cvar-bounded-proof-a}% proof of \ref{eq:cvar-bounded-theorem-a}.

\vspace{3mm}
\textbf{Bounding $I_{1}$}

Note that $k \leq \fnb$. Let's bound $\pbb(A)$ for this case:
\begin{align*}
    \pbb(\calpn(X) \geq \calp(X) + \varepsilon | \knb = k) &\leq \pbb \Big( \frac{1}{\fnb} \sum_{i=0}^{\fnb} X_{[i]} \geq \calp(X) + \varepsilon | \knb = k\Big) ~(\text{using \ref{eq:wang-b}} ) \\
    & \leq \pbb \Big(\frac{1}{k} \sum_{i=1}^{k} X_{[i]} \geq \calp(X) + \varepsilon | \knb  = k \Big) ~(\because f(\cdot) \text{ is decreasing}) \\
    & = \pbb \Big(\frac{1}{k} \sum_{i=1}^{k} \Tilde{X_{i}} \geq \calp(X) + \varepsilon \Big) \\
    & \leq \text{exp}\Big( \frac{k(\varepsilon/b)^{2}}{2 + 2\varepsilon/(3b)} \Big) ~(\text{Using Bernstein's Inequality})
\end{align*}
Let's bound $I_{1}$ now:
\begin{align*}
    I_{1} &\leq \sum_{k=0}^{\fnb} \binom{n}{k} \bigg(\beta \text{exp}\Big( \frac{(\varepsilon/b)^{2}}{2 + 2\varepsilon/(3b)} \Big)\bigg)^{k} (1-\beta)^{n-k} \\
    & \leq \bigg(1 - \beta \bigg(1 - \text{exp}\Big( \frac{(\varepsilon/b)^{2}}{2 + 2\varepsilon/(3b)} \Big)\bigg)\bigg)^{n} \\
    & \leq \text{exp}\bigg(-n\beta \bigg(1 - \text{exp}\Big( \frac{(\varepsilon/b)^{2}}{2 + 2\varepsilon/(3b)} \Big)\bigg) \bigg) ~(\because e^{x} \geq 1+x) \\
    & \leq \text{exp}\bigg(-n\beta \frac{(\varepsilon/b)^{2}}{5 + 5\varepsilon/(3b)} \bigg)
\end{align*}
The last step is the same as that used for bounding $I_{1}$ in the previous proof.

\textbf{Bounding $I_{2}$}: 

Note that $k \geq \cnb$. Let's begin by bounding $\pbb(A)$:
\begin{align*}
    \pbb(\calpn(X) \geq \calp(X) + \varepsilon | \knb = k) &\leq \pbb \Big( \frac{1}{n\beta} \sum_{i=1}^{\cnb} X_{[i]} \geq \calp(X) + \varepsilon | \knb = k \Big) ~(\text{using \ref{eq:wang-a}}) \\
    & \leq \pbb \Big(\frac{1}{n\beta}\sum_{i=1}^{k} X_{[i]} \geq \calp(X) + \varepsilon | \knb = k \Big) ~(\because k \geq \cnb) \\
    & = \pbb \Big( \frac{1}{k} \sum_{i=1}^{k} X_{[i]} \geq \frac{n\beta}{k} (\calp(X)+\varepsilon) | \knb = k \Big) \\
    & \leq \pbb \Big( \frac{1}{k} \sum_{i=1}^{k} X_{[i]} \geq \calp(X) + \frac{n\beta\varepsilon}{k} - \Big(1 - \frac{n\beta}{k} \Big)b \Big| \knb = k \Big)
\end{align*}
Let $\varepsilon_{1}(k) = \frac{n\beta\varepsilon}{k} - \Big(1 - \frac{n\beta}{k} \Big)b = b \Big( (1+\frac{\varepsilon}{b})\frac{n\beta}{k} - 1 \Big)$. Notice that $\varepsilon_{1}(k) \geq 0$ if $k \leq  (1+\frac{\varepsilon}{b}) n\beta $.

Unlike \ref{cvar-bounded-proof-a}, we can consider the entire range $\varepsilon \in [0,2b]$.

\textbf{Case 1.1} If $\varepsilon$ is very small such that $(1+\frac{\varepsilon}{b}) n\beta  \leq \cnb$, then $\varepsilon_{1}(k) \leq 0$. 

Let's bound $I_{2}$ in this case:
\begin{align*}
    I_{2} & \leq \sum_{k=\cnb}^{n} \pbb(\knb = k) \\
    & = \pbb(\knb \geq \cnb) \\
    & \leq \pbb\Big(\knb \geq (1 + \varepsilon/b)n\beta \Big) = \cnb) \\
    & \leq \text{exp} \Big(-n\beta \frac{(\varepsilon/b)^{2}}{3} \Big) ~(\text{Chernoff on }\knb)
\end{align*}

\textbf{Case 1.2} $ (1+\frac{\varepsilon}{b}) n\beta  > \cnb$

We choose $\ksg = (1+\frac{\gamma\varepsilon}{b})n\beta$ for some $\gamma \in [0,1]$. Note that $ (1+\frac{\varepsilon}{b}) n\beta \geq \ksg \geq n\beta$. Assume that $k^{*}_{\gamma} > \cnb$. The proof when $k^{*}_{\gamma} \leq \cnb$ easily follows. We'll also see that the bound on $I_{2}$ is looser when $k^{*}_{\gamma} > \cnb$. 

Note that $\varepsilon_{1}(k)$ decreases as $k$ increases. Now,
\begin{align*}
    \pbb\Big(\frac{1}{k} \sum_{i=1}^{k} X_{[i]} \geq \calp(X) + \varepsilon_{1}(k) \Big) &= \pbb\Big(\frac{1}{k} \sum_{i=1}^{k} \Tilde{X}_{i} \geq \calp(X) + \varepsilon_{1}(k) | \knb = k \Big) \\
    & \leq \begin{cases}
    \text{exp} \Big(-k \frac{(\varepsilon_{1}(k)/b)^{2}}{2 + 2\varepsilon_{1}(k)/(3b)} \Big); ~\cnb \leq k < k^{*}_{\gamma} \\
    1; \qquad \qquad\qquad\qquad\qquad\qquad\quad k\geq k^{*}_{\gamma}
    \end{cases} \\
    & \overset{\mathclap{\text{(a)}}}{\leq} \begin{cases}
    \text{exp} \Big(-k \frac{(1-\gamma)^{2}(\varepsilon/b)^{2}}{2(1+\gamma\varepsilon/b)^{2} + 2(1+\gamma\varepsilon/b)(1-\gamma)\varepsilon/(3b)} \Big); ~\cnb \leq k < \ksg \\
    1; \qquad \qquad\qquad\qquad\qquad\qquad\qquad\qquad\qquad\qquad ~~ k \geq \ksg
    \end{cases} \\
    & \overset{\mathclap{\text{(b)}}}{\leq} \begin{cases} 
    \text{exp} \Big(-k \frac{(1-\gamma)^{2}(\varepsilon/b)^{2}}{2(1+2\gamma)^{2} + 2(1+2\gamma)(1-\gamma)\varepsilon/(3b)} \Big); ~\cnb \leq k < \ksg \\
    1;\qquad \qquad\qquad\qquad\qquad\qquad\qquad\qquad\qquad~~~ k \geq \ksg
    \end{cases}
\end{align*}
(a) above follows because $\frac{(\varepsilon_{1}(k)/b)^{2}}{2 + 2\varepsilon_{1}(k)/(3b)}$ is an increasing function of $\varepsilon_{1}(k)$ and $\varepsilon_{1}(k)$ decreases as $k$ increases.

(b) above follows because $(1+\gamma\varepsilon/b) \leq (1+2\gamma)$.

Now, we'll bound $I_{2}$:
\begin{align*}
    I_{2} & \leq \sum_{k=\cnb}^{n} \pbb(\knb = k) \pbb\Big(\frac{1}{k} \sum_{i=1}^{k} \Tilde{X}_{i} \geq \calp(X) + \varepsilon_{1}(k) | \knb = k \Big) \\
    & \leq \underbrace{\sum_{k = \ceil{\ksg}}^{n} \binom{n}{k} \beta^{k} (1-\beta)^{n-k}}_{I_{2,a}} + \underbrace{\sum_{k = \cnb}^{\floor{\ksg}} \binom{n}{k}\beta^{k}\text{exp} \Big(-k \frac{(1-\gamma)^{2}(\varepsilon/b)^{2}}{8(1+2\gamma)^{2} + 2(1+2\gamma)(1-\gamma)\varepsilon/(3b)} \Big)(1-\beta)^{n-k}}_{I_{2,b}}
\end{align*}
 
 Let's bound $I_{2,a}$ first. This is very similar to \textit{Case 1.1}.
\begin{align*}
    I_{2,a}  &= \pbb(\knb \geq \ksg) \\
    &\leq \pbb \big(\knb \geq (1+\gamma\varepsilon/b)n\beta\big) \\
    &\leq \text{exp} \Big(-n\beta \frac{\gamma^{2}(\varepsilon/b)^{2}}{3} \Big) ~(\text{Bernstein on } \knb)
\end{align*}

If $\floor{\ksg} < \cnb$, then $I_{2,b} = 0$. 
When $\floor{\ksg} \geq \cnb$, let's bound $I_{2,b}$:
\begin{align*}
    I_{2,b} &\leq \bigg(1 - \beta\Big(1 - \text{exp} \Big(- \frac{(1-\gamma)^{2}(\varepsilon/b)^{2}}{2(1+2\gamma)^{2} + 2(1+2\gamma)(1-\gamma)\varepsilon/(3b)} \Big) \Big)\bigg) \\
    &\leq \text{exp} \bigg(-n\beta\Big(1 - \text{exp} \Big(- \frac{(1-\gamma)^{2}(\varepsilon/b)^{2}}{2(1+2\gamma)^{2} + 2(1+2\gamma)(1-\gamma)\varepsilon/(3b)} \Big) \Big) \bigg)
\end{align*}
We know $1 - e^{-x} \geq x - x^{2}/2 = x(1-x/2)$. Now,
\begin{align*}
    g(\varepsilon, \gamma) &= 1 - \frac{1}{2}\frac{(1-\gamma)^{2}(\varepsilon/b)^{2}}{2(1+2\gamma)^{2} + 2(1+2\gamma)(1-\gamma)\varepsilon/(3b)} \\
    &\overset{\mathclap{\text{(a)}}}{\geq} 1 - \frac{1}{2}\frac{2(1-\gamma)^{2}}{(1+2\gamma)^{2} + 2(1+2\gamma)(1-\gamma)/3} \\
    &\overset{\mathclap{\text{(b)}}}{\geq} \frac{2}{5}
\end{align*}
(a) above follows  because $g(\varepsilon, \gamma)$ decreases with $\varepsilon$. We put $\varepsilon = 2b$.

(b) above follows because increase in $\gamma$ increases the RHS of second step. Hence, we put $\gamma = 0$. Irrespective of whether $\floor{\ksg}<\cnb$ or $\floor{\ksg}\geq \cnb$ :
\begin{align*}
    I_{2,b} \leq \text{exp} \bigg(-n\beta\frac{2}{5} \frac{(1-\gamma)^{2}(\varepsilon/b)^{2}}{2(1+2\gamma)^{2} + 2(1+2\gamma)(1-\gamma)\varepsilon/(3b)} \bigg)
\end{align*}
Bounding $I_{2}$ for this case, we get
\begin{align*}
    I_{2} &\leq I_{2,a} + I_{2,b} \\
    & \leq \text{exp} \Big(-n\beta \frac{\gamma^{2}(\varepsilon/b)^{2}}{3} \Big) + \text{exp} \bigg(-n\beta\frac{2}{5} \frac{(1-\gamma)^{2}(\varepsilon/b)^{2}}{2(1+2\gamma)^{2} + 2(1+2\gamma)(1-\gamma)\varepsilon/(3b)} \bigg)
\end{align*}
$\frac{\gamma^{2}}{3} = \frac{2}{5} \frac{(1-\gamma)^{2}}{2(1+2\gamma)^{2}}$Here, take $\gamma = 0.3206$.
\begin{align*}
    I_{2} &\leq \text{exp} \big(-0.1028n\beta(\varepsilon/b)^{2}\big) + \text{exp} \bigg(-n\beta \frac{0.1028(\varepsilon/b)^{2}}{1 + 0.1379\varepsilon/b} \bigg) \\
    & \leq 2 \text{exp} \bigg(-n\beta \frac{0.1028(\varepsilon/b)^{2}}{1 + 0.1379\varepsilon/b} \bigg) \\
    &= 2 \text{exp} \bigg(-n\beta \frac{(\varepsilon/b)^{2}}{9.73 + 1.342\varepsilon/b} \bigg)
\end{align*}
The bound obtained on $I_{2}$ in \textbf{Case 1.1} is tighter than the above bound. But we need to take the looser bound because our bound should be valid for all $\varepsilon \in [0,2b]$. Hence, we take the above bound on $I_{2}$.

Finally, we can bound $I$:
\begin{align*}
    I &\leq I_{1} + I_{2} \\
    &\leq \text{exp}\bigg(-n\beta  \frac{(\varepsilon/b)^{2}}{5 + 5\varepsilon/(3b)} \bigg) + 2 \text{exp} \bigg(-n\beta \frac{(\varepsilon/b)^{2}}{9.73 + 1.342\varepsilon/b} \bigg) \\
    & \leq 3 \text{exp} \bigg(-n\beta \frac{(\varepsilon/b)^{2}}{10 + 1.4\varepsilon/b} \bigg)
\end{align*}

%This completes the proof of \ref{cvar-bounded-theorem}.
%%%%%%%%%%%%%%%%%%%%%%%%%%%%%%%%%%%%%

\section{CVaR Concentration for Heavy Tailed Random Variables (Proof of Theorem~\ref{cvar-ht-theorem})}
\label{cvar-ht-proof}
We begin by bounding the bias in CVaR resulting from our
  truncation. Note that when $b > |\valp(X)|,$ $\valp(X) =
  \valp(X^{(b)}).$ Thus, for $b > |\valp(X)|,$
  \begin{align} 
    |\calp(X) - \calp(X^{(b)})| &= \calp(X) - \calp(X^{(b)}) \nn \\
    &=\frac{1}{1-\alpha}\bigg(\mathbb{E}[X\mathbbm{1}\{X \geq
    \valp(X)\}] - \mathbb{E}[X^{(b)} \mathbbm{1}\{X \geq \valp(X)\}]\bigg) \nn \\
    %&= \frac{1}{1-\alpha}\mathbb{E}[(X-X^{(b)})\mathbbm{1}\{X \geq \valp(X)\}] \nn \\
    &= \frac{1}{1-\alpha}\mathbb{E}[X \mathbbm{1}\{|X| > b\} \mathbbm{1}\{X \geq \valp(X)\}] \nn \\
    &\stackrel{(a)}= \frac{1}{1-\alpha}\mathbb{E}[X \mathbbm{1}\{X >
    b\}] \stackrel{(b)}\leq
    \frac{B}{(1-\alpha)b^{p-1}}. \label{eq:cvar_bias_bound}
\end{align}
Here, ($a$) is a consequence of $b > |\valp(X)|.$ The bound ($b$)
follows from $$\mathbb{E}[X \mathbbm{1}\{X > b\}] \leq
\Exp{\frac{X^p}{X^{p-1}}\mathbbm{1}\{X > b\}} \leq \frac{1}{b^{p-1}}
\Exp{|X|^p} \leq \frac{B}{b^{p-1}}.$$

It follows from \eqref{eq:cvar_bias_bound} that for $b >
\max\left(|\valp(X)|,
  \left[\frac{2B}{\Delta(1-\alpha)}\right]^{\frac{1}{p-1}} \right),$
$|\calp(X) - \calp(X^{(b)})| \leq \frac{\Delta}{2}.$ Thus, for $b$
satisfying \eqref{eq:cvar-ht-truncation_bound}, we have
\begin{align*}
  \prob{|\calp(X) - \cvarte(X)| \geq \Delta} &\leq \prob{|\calp(X) -
    \calp(X^{(b)})| + |\calp(X^{(b)}) - \calpn(X^{(b)})| \geq \Delta} \\
  &\stackrel{(a)}\leq \prob{|\calp(X^{(b)}) - \calpn(X^{(b)})| \geq \frac{\Delta}{2}} \\
  &\stackrel{(b)}\leq 6 \text{exp} \bigg(-n(1-\alpha)
  \frac{(\Delta/b)^{2}}{4(10 + 1.6 \Delta/(2b))} \bigg)\\
  &\stackrel{(c)}\leq 6 \text{exp} \bigg(-n(1-\alpha)
  \frac{(\Delta/b)^{2}}{48} \bigg).
\end{align*}
Here, ($a$) follows the bound on $|\calp(X) - \calp(X^{(b)})|$
obtained earlier. To get ($b$), we invoke
Theorem~\ref{cvar-bounded-theorem}. Finally, ($c$) follows since $b >
\Delta/2.$ This completes the proof.

\section{Error Bounds for Generalized Successive Rejects (Proof of Theorem~\ref{ue-prob-of-error} and Theorem~\ref{sr-prob-of-error})}
\label{gsr-theorem-proof}
The probability of error of the generalized successive rejects algorithm can be upper bounded in the following manner. During phase $k$, at least one of the $k$ worst arms is surviving. If the optimal arm $i^{*}$ is dismissed at the end of phase $k$, it means:
\begin{equation*}
	\xi_{1} \meantea_{n_{k}}(i^{*}) + \xi_{2} \calptk(i^{*}) \geq \min_{i \in \{(K),(K-1),\cdots,(K+1-k)\}} \xi_{1} \meantea_{n_{k}}[i] + \xi_{2} \calptk[i]
\end{equation*}

By using the union bound, we get:
\begin{align*}
	p_{e} \leq \sum_{k=1}^{K-1} \sum_{i=K+1-k}^{K} &\pbb(\xi_{1} \meantea_{n_{k}}(i^{*}) + \xi_{2} \calptk(i^{*}) \geq \xi_{1} \meantea_{n_{k}}[i] + \xi_{2} \calptk[i]) \\
	=  \sum_{k=1}^{K-1} \sum_{i=K+1-k}^{K} &\pbb \big( \xi_{1}(\meantea_{n_{k}}(i^{*}) - \mu(i^{*}) - (\meantea_{n^{k}}[i] - \mu[i] ) ) \\
	+ &\xi_{2}(\calptk(i^{*}) - \calp(i^{*}) - (\calptk[i] - \calp[i]) ) \geq \Delta[i] \big) \\
	\leq \sum_{k=1}^{K-1} \sum_{i=K+1-k}^{K} &\pbb (\xi_{1}(\meantea_{n_{k}}(i^{*}) - \mu(i^{*})) \geq \Delta[i]/4) + \pbb(\xi_{1}(\mu[i] - \meantea_{n_k}[i] ) \geq \Delta[i]/4) \\
	+ & \pbb(\xi_{2}(\calptk(i^{*}) - \calp(i^{*})) \geq \Delta[i]/4 ) + \pbb( \xi_{2}(\calp[i] - \calptk[i]) \geq \Delta[i]/4 ) 
\end{align*}
We've assumed that all the arms satisfy \textbf{C2}. For each arm $i$, we have high probability bounds for $|\meantea_{n}(i) - \mu(i)|$ and $|\calpt(i) - \calp(i)|$ in terms of arm independent parameters $B$ and $p$, we can upper bound $p_{e}$ as follows:
\begin{align*}
  	p_{e} \leq \sum_{k=1}^{K-1} k \big[&\pbb (|\meantea_{n_{k}}(\cdot) - \mu(\cdot)| \geq \Delta[K+1-k]/(4\xi_{1})) \\
	+ &\pbb(|\calptk(\cdot) - \calp(\cdot)| \geq \Delta[K+1-k]/(4\xi_{2}))\big]
  \end{align*}  
By bounding $\pbb (|\meantea_{n}(\cdot) - \mu(\cdot)| \geq \Delta)$ and $\pbb(|\calptk(\cdot) - \calp(\cdot)| \geq \Delta)$, we can bound $p_{e}$.

The statements of Theorem~\ref{ue-prob-of-error} and
Theorem~\ref{sr-prob-of-error} follows easily from the following two
lemmas:
\begin{lemma}
	\label{obl-mean-min-bounds}
	By setting the truncation parameter as $b = n^{q}$ where $q>0$,
	\begin{align*}
		&\pbb(|\mu(k) - \tea_{n}(k)| \geq \Delta) \leq 2\text{exp}\Big(-n^{1-q}\frac{\Delta}{4}\Big) \text{ for } n> n^{*}, \text{ where}\\
		&n^{*} = \Big(\frac{3B}{\Delta}\Big)^{\frac{1}{q\min(1,p-1)}}
	\end{align*}
\end{lemma}
 
\begin{lemma}
	\label{obl-cvar-min-bounds}
    By setting the truncation parameter as $b = n^{q}$ where $q>0$,
    \begin{align*}
        &\pbb(|\calp(X) - \calpt(X)| \geq \Delta) \leq 
        6 \text{exp} \Big(-n^{1-2q} \frac{\beta \Delta^{2}}{48}\Big) \text{ for } n > n^{*}, \text{ where} \\
        &n^{*} = \max \bigg(\Big(\frac{2B}{\beta\Delta}\Big)^{\frac{1}{q(p-1)}},
        \Big(\frac{B}{\text{min}(\alpha,\beta) } \Big)^{\frac{1}{qp}}, 
        \Big(\frac{\Delta}{2} \Big)^{\frac{1}{q}} \bigg)
    \end{align*}
\end{lemma}

\subsection{Proof of Lemma~\ref{obl-mean-min-bounds}}
We'll use the following lemma to prove results for mean minimization
\begin{lemma}
	\label{tea-conc-lemma}
	Assume that $\{X_{i}\}_{i=1}^{n}$ be $n$ I.I.D. samples drawn from the distribution of $X$ which satisfies condition \textbf{C2}, then with probability at least $1-\delta$, 
	\begin{align*}
		|\mu(k) - \tea_{n}(k)| \leq
		\begin{cases}
			\frac{\sum_{i=1}^{n} B/b_{i}^{p-1}}{n} +  \frac{2 b_{n} \text{log}(2/\delta)}{n} + \frac{B}{2b_{n}^{p-1}}; \quad p \in (1,2] \\
			\frac{\sum_{i=1}^{n} B/b_{i}^{p-1}}{n} + \frac{2 b_{n} \text{log}(2/\delta)}{n} + \frac{B^{2/p}}{2 b_{n}}; \quad p \in (2,\infty)
		\end{cases}
	\end{align*}
\end{lemma}
It is adapted from proof of Lemma 1 in \cite{yu2018}.

\textbf{Case 1} $p \in (1,2]$
Using Lemma~\ref{tea-conc-lemma}, if $p \in (1,2]$:
\begin{align*}
	|\mu(k) - \tea_{n}(k)| \leq &\frac{\sum_{i=1}^{n} B/b_{i}^{p-1}}{n} 
	+  \frac{2 b_{n} \log(2/\delta)}{n} + \frac{B}{2b_{n}^{p-1}} \\
	& \leq \frac{3B}{2n^{q(p-1)}} + \frac{2}{n^{1-q}}\log(2/\delta) 
\end{align*}

We want to find $n^{*}$ such that for all $n > n^{*}$:
\begin{align*}
	\underbrace{\frac{3B}{2n^{q(p-1)}}}_{T_{1}} + \underbrace{\frac{2}{n^{1-q}}\log(2/\delta)}_{T_{2}} < \Delta
\end{align*}

Sufficient condition to ensure the above inequality is to make the $T_{1} < \Delta/2$ and $T_{2} \leq \Delta/2$.

$T_{1} \leq \Delta/2$ if:
\begin{equation*}
	n > \Big(\frac{3B}{\Delta}\Big)^{\frac{1}{q(p-1)}}
\end{equation*}

Equating $T_{2} = \Delta/2$, we get:
\begin{equation*}
	\delta = 2 \text{exp}\Big(-n^{1-q} \frac{\Delta}{4}\Big)
\end{equation*}

\textbf{Case 2} $p \in (2,\infty)$

Using Lemma~\ref{tea-conc-lemma}, if $p \in (2,\infty)$:
\begin{align*}
	|\mu(k) - \tea_{n}(k)| \leq &\frac{\sum_{i=1}^{n} B/b_{i}^{p-1}}{n} 
	+ \frac{2 b_{n} \log(2/\delta)}{n} + \frac{B^{2/p}}{2 b_{n}} \\
	\leq & \frac{B}{n^{q(p-1)}} + \frac{B}{2n^{q}} +  \frac{2 \log(2/\delta)}{n^{1-q}} \\
	\leq & \frac{3B}{2n^{q}} + \frac{2 \log(2/\delta)}{n^{1-q}}
\end{align*}

We want to find $n^{*}$ such that for all $n > n^{*}$:
\begin{align*}
	\underbrace{\frac{3B}{2n^{q}}}_{T_{1}} + \underbrace{\frac{2 \log(2/\delta)}{n^{1-q}}}_{T_{2}} < \Delta \\
\end{align*}

Sufficient condition to ensure the above inequality is to make the $T_{1} < \Delta/2$ and $T_{2} \leq \Delta/2$.

$T_{1} < \Delta/2$ if:
\begin{equation*}
	n > \Big(\frac{3B}{\Delta}\Big)^{\frac{1}{q}}
\end{equation*}

Equating $T_{2} = \Delta/2$, we get:
\begin{equation*}
	\delta = 2 \text{exp}\Big(-n^{1-q} \frac{\Delta}{4}\Big)
\end{equation*}	

\subsection{Bounding Magnitude of VaR}
\label{var-mag-bound}
Before we prove Lemma~\ref{obl-cvar-min-bounds}, we'll first bound $|\valp(X)|$ in terms of $B$, $p$ and $\alpha$.
\begin{lemma}
	\label{mod-var-bound}
	\begin{equation*}
		|\valp(X)| \leq \Big(\frac{B}{\text{min}(\alpha,\beta) } \Big)^{\frac{1}{p}}
	\end{equation*}
\end{lemma}

\begin{proof}
If $\valp(X)>0$, by definition:
\begin{align*}
	1 - \alpha = &\int_{\valp(X)}^{\infty} dF_{X}(x)  \\
			   = &\int_{\valp(X)}^{\infty} |x|^{p}/|x|^{p} dF_{X}(x) \\
			   \leq &B/|\valp(X)|^{p}
\end{align*}
Hence, $|\valp(X)| \leq (\frac{B}{\beta})^{\frac{1}{p}}$.

If $\valp(X)<0$, by definition:
\begin{align*}
	\alpha = &\int_{-\infty}^{\valp(X)} dF_{X}(x) \\
		   = &\int_{-\infty}^{\valp(X)} |x|^{p}/|x|^{p} dF_{X}(x) \\
		   \leq &B/|\valp(X)|^{p}
\end{align*}
Hence, $|\valp(X)| \leq (\frac{B}{\alpha})^{\frac{1}{p}}$.	
\end{proof}

\subsection{Proof of Lemma~\ref{obl-cvar-min-bounds}}
The proof follows from Theorem~\ref{cvar-ht-theorem} and Lemma~\ref{mod-var-bound}. We're growing our truncation parameter as $n^{q}$. Therefore, 
\begin{align*}
    b = n^{q} > \max\left(\frac{\Delta}{2}, 
    \Big(\frac{B}{\text{min}(\alpha,\beta) } \Big)^{\frac{1}{p}},
    \left[\frac{2B}{\Delta(1-\alpha)}\right]^{\frac{1}{p-1}} \right)
\end{align*}

\section{Error Bounds for Non-oblivious Algorithms}
\label{nonoblivious-proofs}
In the non-oblivious setting, error bounds for the generalized successive rejects algorithm follow from the following two lemmas.
\begin{lemma}
	\label{non-obl-mean-min-bounds}
	By setting the truncation parameter $b > \Big(\frac{3B}{\Delta}\Big)^{\frac{1}{\min(1,p-1)}}$,
	\begin{align*}
		&\pbb(|\mu(k) - \tea_{n}(k)| \geq \Delta) \leq 2\text{exp}\Big(-n \frac{\Delta}{4b} \Big).
	\end{align*}
\end{lemma}
\begin{lemma}
	\label{non-obl-cvar-min-bounds}
	By setting the truncation parameter $b > \max\left(\frac{\Delta}{2}, 
	\Big(\frac{B}{\text{min}(\alpha,\beta) } \Big)^{\frac{1}{p}},
    \left[\frac{2B}{\Delta(1-\alpha)}\right]^{\frac{1}{p-1}} \right)$,
    \begin{align*}
    	\prob{|\calp(k) - \cvarte(k)| \geq \Delta} \leq 6 \text{exp} 
			\bigg(-n(1-\alpha)\frac{\Delta^2}{48 b^2}\bigg).
    \end{align*}  	
\end{lemma}
Note that the truncation parameters here are not a function of $n$ and therefore we get an exponentially decaying bound.

\subsection{Proof of Lemma~\ref{non-obl-mean-min-bounds}}
\begin{lemma}
	By setting the truncation parameter $b > \Big(\frac{3B}{\Delta}\Big)^{\frac{1}{\min(1,p-1)}}$ where $q>0$,
	\begin{align*}
		&\pbb(|\mu(k) - \tea_{n}(k)| \geq \Delta) \leq 2\text{exp}\Big(-n \frac{\Delta}{4b} \Big)
	\end{align*}
\end{lemma}

\begin{proof}
Using Lemma~\ref{tea-conc-lemma}, by fixing the truncation parameter as $b$, and making simplifications, with probability $1-\delta$, we have:
\begin{align*}
	|\mu(k) - \tea_{n}(k)| \leq
	\begin{cases}
		\frac{3B}{2b^{p-1}} +  \frac{2b\text{log}(2/\delta)}{n} ; \quad p \in (1,2] \\
		\frac{3B}{2b} + \frac{2b\text{log}(2/\delta)}{n} ; \quad p \in (2,\infty)
	\end{cases}
\end{align*}

\textbf{Case 1} $p \in (1,2]$
We're interested to find $b$ and $\delta$ such that for all values of $n$:
\begin{align*}
	\underbrace{\frac{3B}{2b^{p-1}}}_{T_{1}} +  \underbrace{\frac{2b\log(2/\delta)}{n}}_{T_{2}} < \Delta \\
\end{align*}

A sufficient condition for the above equation to be valid is $T_{1} < \Delta/2$ and $T_{2} = \Delta/2$.

To ensure $T_{1} < \Delta/2$, take $b > \Big(\frac{3B}{\Delta}\Big)^{\frac{1}{p-1}}$. 

By equating $T_{2} = \Delta/2$, we get $\delta = 2\text{exp}\Big(-n \frac{\Delta}{4b} \Big)$ where $b$ is what we found above.

\textbf{Case 2} $p \in (2,\infty)$
We're interested to find $b$ and $\delta$ such that for all values of $n$:
\begin{align*}
	\underbrace{\frac{3B}{2b} }_{T_{1}} +  \underbrace{\frac{2b\log(2/\delta)}{n}}_{T_{2}} < \Delta \\
\end{align*}

A sufficient condition for the above equation to be valid is $T_{1} < \Delta/2$ and $T_{2} = \Delta/2$.

To ensure $T_{1} = \Delta/2$, take $b > \frac{3B}{\Delta}$. 

By equating $T_{2} = \Delta/2$, we get $\delta = 2\text{exp}\Big(-n \frac{\Delta}{4b} \Big)$ where $b$ is what we found above.	
\end{proof}

\subsection{Proof of Lemma~\ref{non-obl-cvar-min-bounds}}
Lemma~\ref{non-obl-cvar-min-bounds} follows from Theorem~\ref{cvar-ht-theorem} and Lemma~\ref{mod-var-bound}.

\end{document}